\documentclass[nohyperref]{article}

\usepackage{microtype}
\usepackage{graphicx}
\usepackage{subfigure}
\usepackage{booktabs} 
\usepackage{multirow}
\usepackage{siunitx}
\usepackage{rotating}
\usepackage{hyperref}

\usepackage[accepted]{icml2022}

\usepackage{amsmath}
\usepackage{amssymb}
\usepackage{mathtools}
\usepackage{amsthm}

\usepackage[capitalize,noabbrev]{cleveref}

\theoremstyle{plain}
\newtheorem{theorem}{Theorem}[section]

\theoremstyle{definition}
\newtheorem{definition}[theorem]{Definition}

\theoremstyle{remark}

\usepackage[textsize=tiny]{todonotes}

\icmltitlerunning{Towards Scaling Difference Target Propagation by Learning Backprop Targets}

\begin{document}

\twocolumn[
\icmltitle{Towards Scaling Difference Target Propagation by Learning Backprop Targets}

\icmlsetsymbol{equal}{*}
\begin{icmlauthorlist}
\icmlauthor{Maxence Ernoult}{ibm}
\icmlauthor{Fabrice Normandin}{equal,mila}
\icmlauthor{Abhinav Moudgil}{equal,mila,concordia}
\icmlauthor{Sean Spinney}{mila,udem}
\icmlauthor{Eugene Belilovsky}{mila,concordia}
\icmlauthor{Irina Rish}{mila,udem}
\icmlauthor{Blake Richards}{mila,mcgill,neuro}
\icmlauthor{Yoshua Bengio}{mila,udem}
\end{icmlauthorlist}

\icmlaffiliation{ibm}{IBM Research, Paris. Work done while at Mila.}
\icmlaffiliation{udem}{UdeM.}
\icmlaffiliation{mila}{Mila.}
\icmlaffiliation{concordia}{Concordia University.}
\icmlaffiliation{mcgill}{McGill University.}
\icmlaffiliation{neuro}{Montreal Neurological Institute}

\icmlcorrespondingauthor{Maxence Ernoult}{maxence.ernoult@ibm.com}
\icmlcorrespondingauthor{Yoshua Bengio}{yoshua.bengio@mila.quebec}

\icmlkeywords{Machine Learning, ICML}

\vskip 0.3in
]
\printAffiliationsAndNotice{\icmlEqualContribution}

\begin{abstract}
The development of biologically-plausible learning algorithms is important for understanding learning in the brain, but most of them fail to scale-up to real-world tasks, limiting their potential as explanations for learning by real brains. As such, it is important to explore learning algorithms that come with strong theoretical guarantees and can match the performance of backpropagation (BP) on complex tasks. One such algorithm is Difference Target Propagation (DTP), a biologically-plausible learning algorithm whose close relation with Gauss-Newton (GN) optimization has been recently established. However, the conditions under which this connection rigorously holds preclude layer-wise training of the feedback pathway synaptic weights (which is more biologically plausible). Moreover, good alignment between DTP weight updates and loss gradients is only loosely guaranteed and under very specific conditions for
the architecture being trained. In this paper, we propose a novel feedback weight training scheme that ensures both that DTP approximates BP and that layer-wise feedback weight training can be restored without sacrificing any theoretical guarantees. Our theory is corroborated by experimental results and we report the best performance ever achieved by DTP on CIFAR-10 and ImageNet 32$\times$32. 
\end{abstract}

\section{Introduction}
\label{sec:introduction}

\begin{figure}[ht!]
  \begin{center}
    \includegraphics[width=0.5\textwidth]{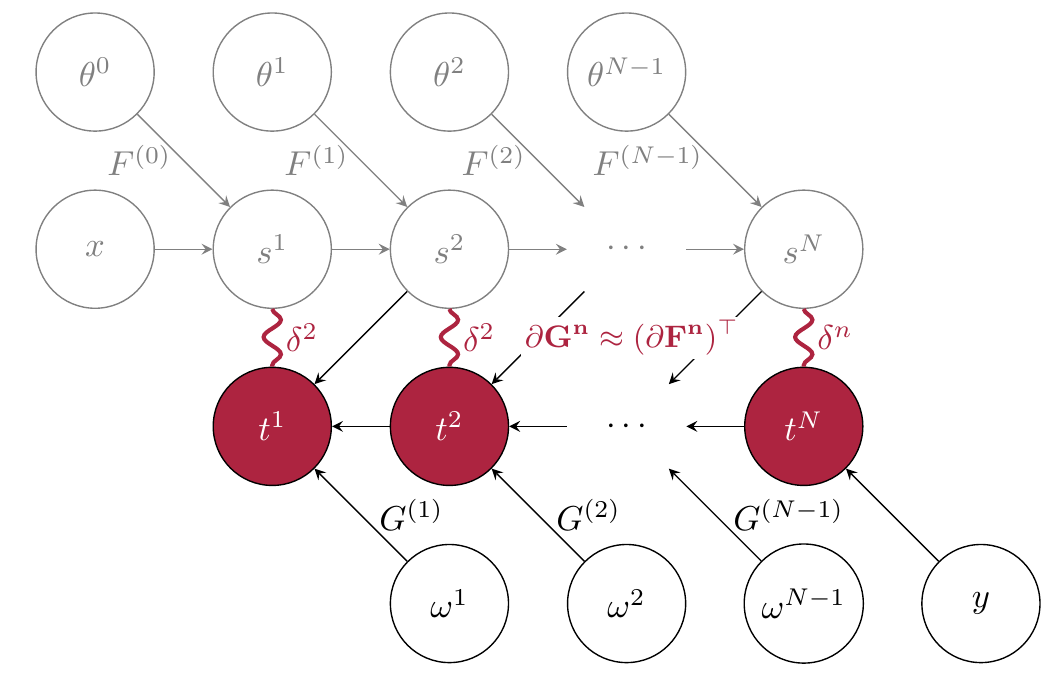}
  \end{center}
  \label{fig:dtp}
  \caption{Computational graph of the feedforward pathway $\mathcal{F}$ (on the left, shaded) with input $x$  and associated DTP feedback pathway with ground-truth label $y$ (right). The \emph{targets} $t^n$ (purple nodes) are forward-propagated through the $G^{n + 1}$ operator whose Jacobian has been made to approximately match that of the transpose of $F^n$. This way, the resulting local activation differences $\delta^n\propto t^n - s^n$ encode backprop error signals. We thus learn to estimate layer-wise \emph{backprop} targets.}
\end{figure}

Although artificial neural networks were originally inspired by the brain, the strict implementation of the backpropagation algorithm (BP) violates biological constraints, and no known biologically plausible candidate algorithm can match its performance on challenging tasks. Conversely, bridging this gap could bring a better understanding of biological learning \cite{richards2019deep}. Recent efforts towards this goal suggest that it could be achieved by developing learning algorithms that relax the requirements of BP while preserving strong theoretical guarantees. 

\emph{Target Propagation} (TP) \cite{yann1987modeles} and its \emph{Difference Target Propagation} (DTP) variants \cite{bengio2014auto, lee2015difference,  bartunov2018assessing, ororbia2019biologically, bengio2020deriving, meulemans2020theoretical} constitute a family of such algorithms which, from the biological prospective, sidesteps two issues of BP. Most importantly, TP computes error signals in feedforward architectures by propagating \emph{target values} for the neurons rather than error gradients, thereby aligning better with our understanding of what feedback pathways in the brain communicate \cite{lillicrap2020backpropagation}. A major consequence of handling neural activation targets across all layers is that feedforward weights can be updated in a fully local fashion to push neural activations closer to their target values. Second, TP routes those targets through a distinct set of feedback weights rather than transporting the weights from the feedforward pathway \cite{lillicrap2016random}. Rather than being fixed throughout learning, though, these feedback weights learn to \emph{invert} the feedforward pathway. But what would it take to learn to \emph{backprop} the feedforward pathway? This is the central question addressed by the present work.

Nevertheless, TP algorithms have yet to scale to complex tasks, and as such, they do not yet stand as a compelling biological learning model.
Recent work highlighting the connection between TP and Gauss-Newton (GN) optimization \cite{bengio2020deriving, meulemans2020theoretical} has incentivized
to revisit the scalability of TP algorithms \cite{bartunov2018assessing}. More precisely, \citet{meulemans2020theoretical} demonstrate that while TP neural activations updates emulates GN optimization in invertible neural networks, this connection can be maintained with DTP on non-invertible networks if the feedback weights training scheme is changed accordingly.
Indeed, to emulate GN optimization, as the pseudo-inverse of the whole feedforward pathway does not factorize as the product of each feedforward module's pseudo-inverse, each feedback module should capture the pseudo-inverse of the \emph{whole} downstream feedforward pathway: when computing the resulting Difference Reconstruction Loss (DRL), noisy perturbations subsequently need to be propagated all the way up to the output layer. However, their approach still has limitations from a biological learning perspective. First, enforcing GN optimization in DTP like this precludes \emph{layer-wise} feedback weights training and instead calls for the use of direct connections in the feedback pathway: this topological restriction seriously compromises biological plausibility. Second, the resulting optimization algorithm used to update the feedforward weights is a hybrid between between gradient descent and GN optimization. Therefore, only loose alignment between backprop and DTP updates can be accounted for by their theory and with restrictive assumptions on the architecture being trained. Finally, although that theory offers a principled way to design the architectures trained by these variants of DTP, the CIFAR-10 training experiments they report are limited to relatively shallow architectures with poor performance.

In this paper, we propose to revisit the GN interpretation of DTP by having the feedback pathway synaptic weights compute layer-wise \emph{BP targets} rather than GN targets. To this end, we propose a novel feedback weights training scheme which, by construction, pushes the Jacobian of the feedback operator towards the transpose of its feedforward counterpart, in a layer-wise fashion and without having to use direct feedback connections in the feedback pathway. Therefore, assuming this condition holds for all the layers and keeping everything else unchanged in the DTP algorithm, the DTP feedforward weight updates closely approach those of BP. This leads us to a scalable biologically plausible approximation of BP.

More specifically, our contributions are as follows:

\begin{itemize}
    \item We propose a novel \emph{Local Difference Reconstruction Loss} (L-DRL) along with an algorithm to train the feedback weights which ensures that the Jacobian of the feedback pathway matches the transpose of the Jacobian of the associated feedforward pathway (Section~\ref{subsec:feedback-loss}, Theorem~\ref{thm:feedback-loss}, Alg.~\ref{alg:L-DRL}). We call this condition the \emph{Jacobian Matching Condition} (JMC, \cref{def:JMC}).
    \item Assuming the JMC holds for a given architecture and using the standard DTP equations to propagate targets, we demonstrate that DTP feedforward weight updates approximate BP gradients (Section~\ref{subsec:feedforward-loss}, Theorem~\ref{thm:feedforward-loss}). We say that such an architecture satisfies the \emph{Gradient Matching Property} (GMP).
    \item We numerically demonstrate the GMP and JMC, showing that L-DRL is more efficient than DRL to align feedforward and feedback weights and that the GMP is subsequently significantly better satisfied 
    (Section~\ref{subsec:jmc}-\ref{subsec:GMP}). 
    \item Finally, we validate our novel implementation of DTP on training experiments on MNIST, Fashion MNIST, CIFAR-10 and  ImageNet 32$\times$32 \cite{NIPS2016_b1301141} (Section~\ref{subsec:train-experiments}). In particular, we achieve a 89.38 \% accuracy on CIFAR-10 and 60.6 \% top-5 accuracy on ImageNet 32$\times$32, which are the best performances ever reported in the DTP literature on these datasets and nearly match the performance of BP on the same architectures.
\end{itemize}

\section{Related Work}

DTP borrows several key concepts from the biologically plausible deep learning literature. First, resorting to a distinct set of weights to route error signals in the feedback pathway as done in DTP solves a problem known as \emph{weight transport} \cite{lillicrap2016random}. While having randomly initialized \emph{fixed} feedback weights is sufficient to carry useful error signals on MNIST \cite{lillicrap2016random,nokland2016direct}, subsequent studies demonstrated it was insufficient to scale to harder tasks \cite{moskovitz2018feedback, bartunov2018assessing, launay2019principled,crafton2019direct}. The main approach undertaken to overcome this issue is to add extra mechanisms to promote alignment between feedforward and feedback weights \cite{xiao2018biologically, lansdell2019learning, guerguiev2019spike, akrout2019deep, kunin2020two}. More specifically, many of these mechanisms are based on the idea of perturbing the feedforward activations with noise, and communicating the resulting noisy activations in the feedback pathway to coordinate feedback and feedforward weight updates consistently \cite{akrout2019deep, lansdell2019learning, kunin2020two}, which constitutes a second important feature of DTP. However, a limitation of many of these algorithms is that they require gradient computation of the operations carried out in the feedforward pathway. One solution to mitigate this issue, which is the third important ingredient of DTP, is to propagate neural activation differences as implicit error signals rather than error gradients \cite{lillicrap2020backpropagation}. These error signals may typically arise from a mismatch between feedforward (bottom-up) predictions and (top-down) actual feedback \cite{whittington2017approximation, NEURIPS2018_dendritic, choromanska2019beyond}, as it is the case for DTP, or from a perturbation from equilibrium \cite{scellier2017equilibrium}. Recent works have explored the application of DTP to recurrent neural networks \cite{manchev2020target, roulet2021target}, albeit with the implausible requirement of processing inputs backward in time during target computation, a challenge that we do not aim to address in the present paper. As emphasized in the introduction, the closest work to ours is that of \cite{meulemans2020theoretical}, and we show the theoretical and experimental advantages of our approach. 

\section{Background}

We first introduce the key notations and assumptions used throughout this paper.

\begin{definition}
We define a \emph{feedforward} architecture as:

\begin{equation}
    \mathcal{F}(x) = F^{N - 1} \circ F^{N-1} \circ \dots \circ F^0(x),
\end{equation}

where each feedforward module $F^{n}(\cdot; \theta^n)$ is parametrized by its feedforward weights $\theta^n$. Each $F^n$ is paired with a feedback module $G^n(\cdot; \omega^n)$ with distinct weights $\omega^n$.
\label{def:archi}
\end{definition}

\begin{definition}
We recursively define the \emph{layers} $s^1$, $\dots$, $s^N$ of an architecture $\mathcal{F}$ defined by \cref{def:archi} as: 

\begin{align}
\left \{
\begin{array}{ll}
 s^0 &= x \\
 s^{n + 1} &=F^n(s^{n};\theta^n)\quad \forall n = 0 \cdots N - 1
\end{array}
\right.
\end{align}

$G^n$ can either take as input the feedforward path activations $s^{n+1}=F^n(s^n)$ (with $G^n$ then forming the decoder part of a kind of auto-encoder with $F^n$ as encoder) or the backward path targets $t^{n+1}$ produced by $G^{n+1}$ from $t^{n+2}$ and representing targets for $s^{n+1}$.
\end{definition}

\paragraph{Learning setting.} We study the supervised context where, given a target $y$, the goal is to find the forward weights $\theta^n$ which minimize a predictive loss $\mathcal{L}_{\rm pred}(s^N, y)$.

\paragraph{Notations. } We denote the Frobenius dot product between two matrices $A$ and $B$ as $\langle A, B \rangle_F \triangleq {\rm Tr}\left(A \cdot B^\top\right)$. Also, we denote $\partial_x F(x_\star) = \frac{\partial F}{\partial x}(x_\star)$ the Jacobian of $F$ with respect to $x$ evaluated at $x_\star$. For notational simplicity, we may omit to write $x_\star$, in which case the Jacobians are implicitly evaluated on the feedforward activations.

\subsection{Difference Target Propagation (DTP)}

Instead of transporting the transpose Jacobian of the feedforward operators $\partial_{s^n} F^{n^\top}$ to the feedback pathway, TP and variants use a separate set of parameters through the feedback operator $G^n$ to carry targets across layers. The $G^n$ operators are subsequently trained, layer-wise, to approximately invert their associated feedforward counterpart: $G^n \approx \left(F^n\right)^{-1}$. TP learning thus entangles feedforward \emph{and} feedback weights training.

\paragraph{Forward weights training.}Target values for the neurons should be such that they decrease the predictive loss $\mathcal{L}_{\rm pred}$. For most TP algorithms, the first target is computed as:

\begin{equation}
  t^N_\beta = s^N -\beta\frac{\partial \mathcal{L}_{\rm pred}}{\partial s^N},
  \label{eq:first-target}  
\end{equation}

where $\beta$ is a small nudging parameter. 
In TP, the subsequent upstream targets are \emph{propagated} the the feedback operators as $t^{n}_\beta = G^n(t^{n+1}_\beta;\omega^n)$. However in non-invertible feedforward networks, this results in a significant reconstruction error $s^n - G^n(s^n;\omega^n)$, which was shown to compromise learning. \emph{Difference} Target Propagation \cite{lee2015difference} aims to solve this issue by removing this reconstruction term from the target computation:

\begin{equation}
    t^{n}_\beta =  G^{n}(t^{n + 1}_\beta; \omega^{n}) + s^{n} - G^{n}(s^{n + 1}; \omega^{n}). \label{eq:target-propagation}
\end{equation}

For later convenience, we denote 
\begin{equation}
\tilde{G}(t^{n+1}, s^{n+1};\omega^n) \triangleq G^{n}(t^{n + 1}_\beta; \omega^{n}) + s^{n} - G^{n}(s^{n + 1}; \omega^{n})
\label{eq:tilde-G}
\end{equation}
the feedback operation used to propagate the targets in Eq.~(\ref{eq:target-propagation}). Finally, the parameters $\theta^n$ are updated by the local loss $\mathcal{L}^n_\theta$, defined as:

\begin{equation}
    \mathcal{L}_{\theta}^n = \frac{1}{2\beta}\|t_\beta^n - s^n \|^2,
    \label{eq:def-ff-loss}
\end{equation}

where $t^n_\beta$ is treated as a constant and the gradients blocked at $s^{n-1}$. For example, if $F^n$ was linear, we would have the weight update $\Delta \theta^n \propto s^{n-1}\cdot (t^n_\beta - s^n)^\top$.

\paragraph{Feedback weights training.} Both TP and DTP employ the same mechanism to train the $G^n$ operators. First, the feedforward activations $s^n$ get a noisy perturbation $\epsilon$. The resulting noisy activations $s^n_\epsilon$ perturb the next layer $s^{n + 1}_\epsilon$ through $F^n$, which
in return yields a noisy reconstruction $r^n_\epsilon$ through $G^n$. The feedback weights are then updated to minimize the local loss $\mathcal{L}^n_\omega$ defined as:
\begin{equation}
\hat{\mathcal{L}}_{\omega}^n = \frac{1}{2} \|r_\epsilon^n - s^n_\epsilon\|^2,
\label{def:vanilla-dtp-fb-loss}
\end{equation}
where $s^n$ is treated as a constant and the gradient are blocked at $s^{n+1}$. Assuming again a linear $G^n$, the resulting feedback weight update also reads in a local fashion: $\Delta \omega^n \propto s^{n+1}\cdot(r^n_\epsilon - s^n_\epsilon)^\top$.

\begin{algorithm}[ht!]
    \caption{Standard DTP feedback weight training \\ \cite{lee2015difference}}
    \label{alg:vanilla-dtp-fb}
        \begin{algorithmic}[1]
         \STATE $\epsilon \sim \mathcal{N}(0, \sigma^2)$, $s^n_\epsilon=s^n + \epsilon$
         \STATE $s^{n+1}_\epsilon = F^n(s^n_\epsilon;\theta^n)$
         \STATE $r^n_{\epsilon} = G^n(s^{n+1}_\epsilon;\omega^n)$
         \STATE Update $\omega^n$ with $\hat{\mathcal{L}}_{\omega}^n = \frac{1}{2} \|r_\epsilon^n - s^n_\epsilon\|^2$.
    \end{algorithmic}
\end{algorithm}

\subsection{Connection between DTP and Gauss-Newton Optimization}

Using Eq.~(\ref{eq:first-target})-(\ref{eq:target-propagation}) and sending $\beta \to 0$, note that the DTP activation updates can be conveniently defined as:

\begin{equation}
    \delta_{\rm DTP}^n \triangleq  \lim_{\beta \to 0} \frac{t^n_\beta - s^n}{\beta} = - \left[\prod_{k=n}^{N-1} \partial_{s^{k+1}}G^k\right]\cdot \frac{\partial \mathcal{L}_{\rm pred}}{\partial s^N}
    \label{eq:def-delta-DTP}
\end{equation}

It was suggested that under some conditions, $\delta^n_{\rm DTP}$ encoded Gauss-Newton updates \cite{gauss1877theoria} of the layer activations with respect to the output loss function \cite{bengio2020deriving, meulemans2020theoretical}. In \emph{invertible} networks, i.e. assuming $\left(F^n\right)^{-1}$ exists, the Gauss-Newton update of layer $s^n$ with respect to $\mathcal{L}_{\rm pred}$ is:

\begin{equation}
    \delta^n_{\rm GN} = - \left[\partial_{s^n}\overline{F}^n \right]^{-1}\cdot \frac{\partial \mathcal{L}_{\rm pred}}{\partial s^N},
    \label{eq:invertible-GN-updates}
\end{equation}

where $\overline{F}^n = F^N \circ \cdots \circ F^n$ denotes the forward mapping from $s^n$ to $s^N$. Furthermore, assuming $G^n = F^{{n-1}^{-1}}$ for all $n=1\cdots N-1$, from Eq.~(\ref{eq:def-delta-DTP}), Eq.~(\ref{eq:invertible-GN-updates}) and the inverse function theorem, it can be seen that $\delta^n_{\rm GN}=\delta^n_{\rm DTP}$.

In \emph{non-invertible} networks, \citet{meulemans2020theoretical} show that with a block-diagonal approximation of the Gauss-Newton curvature matrix,
the Gauss-Newton update of $s^n$ with respect to $\mathcal{L}_{\rm pred}$ reads:

\begin{equation}
    \delta^n_{\rm GN} = - \left[\partial_{s^n}\overline{F}^n \right]^{\dagger}\cdot \frac{\partial \mathcal{L}_{\rm pred}}{\partial s^N},
    \label{eq:non-invertible-GN-updates}
\end{equation}

where $A^\dagger = \lim_{\lambda \to 0} A^\top \cdot \left(A\cdot A^\top - \lambda \right)^{-1}$ denotes the \emph{Moore-Penrose} pseudo-inverse.
However, there are are two reasons why in this case we may not have $\delta^n_{\rm GN}=\delta^n_{\rm DTP}$. First, using Eq.~(\ref{def:vanilla-dtp-fb-loss}) as a reconstruction loss, it may not hold in general that $\partial_{s^{n+1}}G^n = \left(\partial_{s^{n}}F^n \right)^\dagger$. Second, even assuming this condition holds, $\left[\partial_{s^n}\overline{F}^n \right]^{\dagger}$ generally does not factorize as $\prod_{k=n}^N\left[\partial_{s^k}F^k \right]^{\dagger}$. A direct consequence of this is that DTP standard layer-wise feedback weights training leads to mostly inefficient feedforward weight updates that fail to move the output layer towards its target.

\citet{meulemans2020theoretical} show that by adapting DTP standard feedback training scheme, $\delta^n_{\rm GN}=\delta^n_{\rm DTP}$ can be recovered for non-invertible networks. Instead of propagating perturbated activations $s_\epsilon^n$ back and forth through $F^n$ and $G^n$ into the reconstructed activation $r^n_\epsilon$ to train $\omega^n$, they prescribe sending $s_\epsilon^n$ up to $\hat{y}$ through $\overline{F}^n$, back into $r_\epsilon^n$ through $\tilde{G}^n \circ \cdots \circ \tilde{G}^N$ where $\tilde{G}^n$ (Eq.~(\ref{eq:tilde-G})) stands for the operator used for the target computation in Eq.~(\ref{eq:target-propagation}). Finally, an extra noisy perturbation in the output layer $s^N + \eta$ needs to be propagated back into $r^n_\eta$. The resulting \emph{Difference Reconstruction Loss} (DRL) to be optimized is defined as:

\begin{equation}
  \hat{\mathcal{L}}_\omega^n = \frac{1}{2}\|r_\epsilon^n - s^n_\epsilon\|^2 + \lambda \|r^n_\eta - s^n\|^2.  
  \label{eq:def-feedback-loss-meulemans}
\end{equation}

In practice though, they replace the second term by weight decay. Taking expectation of Eq.~(\ref{eq:def-feedback-loss-meulemans}) and sending the noise amplitude to 0, it can be shown that minimizing $\hat{\mathcal{L}}^n_\omega$ yields the desired property: $  \prod_{k=n}^{N-1} \partial_{s^{k+1}}G^k = \left(\partial_{s^n}\overline{F}^{n}\right)^\dagger$.
\begin{algorithm}[ht!]
    \caption{Difference Reconstruction Loss (DRL) feedback weight training \cite{meulemans2020theoretical}}
    \label{alg:meulemans-dtp-fb}
        \begin{algorithmic}[1]
         \STATE $\epsilon \sim \mathcal{N}(0, \sigma^2)$, $s^n_\epsilon=s^n + \epsilon$
        \FOR{$k=n \cdots N - 1$}
            \STATE $s^{k+1}_\epsilon = F^{k}(s^{k}_\epsilon;\theta^k)$
        \ENDFOR
        \STATE $r^N_\epsilon = s^N_\epsilon$
        \FOR{$k=N-1 \cdots n$}
            \STATE $r^k_\epsilon = G^k(r^N_\epsilon, r^{k+1}_\epsilon;\omega^k)+s^k - G^k(s^N, s^{k+1};\omega^k)$
        \ENDFOR
        
        \STATE Update $\omega^n$ with $\hat{\mathcal{L}}_\omega^n= \frac{1}{2}\|r_\epsilon^n - s^n_\epsilon\|^2 +
        \lambda \omega^n$.
        
        
    \end{algorithmic}
\end{algorithm}

\section{Learning Backprop Targets rather than Gauss-Newton Targets}

In the spirit of \citet{meulemans2020theoretical}, we propose to adapt the feedback weight training and the reconstruction loss, but we make it so that $G^n$ learns the \emph{transpose Jacobian} of its associated feedforward module $F^n$ rather than its pseudo-inverse. This way, by construction, the DTP weight updates are made to match \emph{BP} weight updates rather than a hybrid between BP and Gauss-Newton updates. We can also avoid the requirement of direct connections and restore layer-wise feedback weights training while preserving theoretical guarantees with respect to BP. 

\subsection{Feedback weights training}
\label{subsec:feedback-loss}

\begin{definition}
\label{def:JMC}
For a given architecture $\mathcal{F}$ defined by \cref{def:archi}, we say that a feedforward module $F^n$ and associated feedback module $G^n$ satisfy the \emph{Jacobian-Matching Condition} (JMC) if:

\begin{equation}
\left(\partial_{s^{n}} F^n(s^n)\right)^\top  = \partial_{s^{n+1}} G^n(s^{n+1})
\label{eq:jmc}
\end{equation}

We say that an architecture $\mathcal{F}$ satisfies the JMC if for $n=1\cdots N$, $(F^n, G^n)$ satisfy the JMC. 
\end{definition}

To illustrate our proposed algorithm to train the feedback weights, let us consider the feedforward module $F^n$ and associated feedback module $G^n$. Let $\epsilon \sim \mathcal{N}(0, \sigma^2)$ be a perturbation to input feature $s^n$ so that the resulting noisy activations $s^n_\epsilon$ triggers a noisy perturbation in the next layer $s^{n+1}_\epsilon$ through $F^n$. Then, we assume $s^{n+1}_\epsilon$ yields in turn a noisy reconstruction $r^n_\epsilon$ through $\tilde{G}^n$ from Eq.~(\ref{eq:tilde-G}) (rather than $G^n$).
Furthermore, we let $\eta \sim \mathcal{N}(0, \sigma^2)$ be a second source of noise in layer $s^{n+1}$. The resulting noisy activations $s^{n+1}_\eta$ create the noisy reconstructions $r^n_\eta$ again through $\tilde{G}^n$. We then prescribe updating the feedback weights with the \emph{Local Difference Reconstruction Loss} (L-DRL) which we define as:

\begin{equation}
    \hat{\mathcal{L}}_\omega^n = - \epsilon^\top \cdot \left(r^n_\epsilon - s^n \right)+ \frac{1}{2}\| r^n_\eta - s^n\|^2
    \label{eq:def-new-fb-loss}.
\end{equation}

\begin{algorithm}[H]
    \caption{Local Difference Reconstruction Loss \\
    (L-DRL) \label{alg:L-DRL}}
        \begin{algorithmic}[1]
        \FOR{i = 1 to K}
         \STATE $s^{n+1} = F^n(s^n;\theta^n)$
         \STATE $\epsilon \sim \mathcal{N}(0, \sigma^2)$, $s^n_\epsilon = s^n + \epsilon$
         \STATE $s^{n+1}_\epsilon = F^n(s^{n}_\epsilon;\theta^n)$
         \STATE $r^n_\epsilon = G^n(s^{n+1}_\epsilon;\omega^n) - G^n(s^{n+1};\omega^n) + s^n$
         \STATE $\eta \sim \mathcal{N}(0, \sigma^2)$, $r^{n+1}_\eta = s^{n+1} + \eta$
         \STATE $r^n_\eta = G^n(r^{n+1}_\eta;\omega^n) - G^n(s^{n+1};\omega^n) + s^n$
         \STATE Update $\omega_n$ to descend layer-wise loss $\mathcal{L}_\omega^n$:
         \STATE \hspace*{1cm}$\hat{\mathcal{L}}_\omega^n = - \epsilon^\top \cdot \left(r^n_\epsilon - s^n \right)+ \frac{1}{2}\| r^n_\eta - s^n\|^2$
        \ENDFOR
    \end{algorithmic}
\end{algorithm}

Contrary to existing DTP approaches, the above procedure is repeated $K$ times per training batch, so that feedback weights can quickly and locally (per-layer) adapt on the fly to the feedforward activations and recent feedforward weight updates. This avoids interleaving phases of \emph{pure} feedback weight training with frozen feedforward weights and instead makes it possible to train feedback and feedforward weights together from the beginning. 

We now state Theorem~\ref{thm:feedback-loss} which guarantees that minimizing $\mathcal{L}^n_{\omega}$ as defined in Eq.~(\ref{eq:def-new-fb-loss}) yields the JMC for layer $n$. 

\begin{theorem}
\label{thm:feedback-loss}
Let: 
\begin{align}
    \widehat{\mathcal{L}}_\omega^n &= -\frac{1}{\sigma^2}\epsilon^\top \cdot \left( r^n_\epsilon - s^n\right)+ \frac{1}{2 \sigma^2}\left\| r_\eta^n - s^n\right\|^2, \\
    \mathcal{L}_\omega^n &= \frac{1}{2}\left \|\partial_{s^n}F^{n^\top} - \partial_{s^{n+1}}G^{n+1} \right\|^2_F.
\end{align}

Then:

\begin{align}
\lim_{\sigma \to 0}\mathbb{E}_{\epsilon, \eta}\left[ \hat{\mathcal{L}}_\omega^n\right] &= -\left \langle \partial_{s^n} F^{n^\top}, \partial_{s^{n+1}}G^n\right\rangle_F \nonumber\\
&+ \frac{1}{2} \left\| \partial_{s^{n+1}}G^n \right \|_F^2\\
\frac{\partial}{\partial \omega}\lim_{\sigma \to 0}\mathbb{E}_{\epsilon, \eta}\left[ \hat{\mathcal{L}}_\omega^n\right] &= \frac{\partial \mathcal{L}_\omega^n}{\partial \omega},
\end{align}
\end{theorem}
This means that training the feedback weights with respect to the local layer loss of Eq.~(\ref{eq:def-new-fb-loss}) makes the feedback path compute the Jacobian of the feedforward path in the limit of small noise and in expectation over the noisy samples.

\subsection{Feedforward weight training}
\label{subsec:feedforward-loss}
Although our new implementation of DTP uses the exact same equations as standard DTP to propagate the targets (Eq.~(\ref{eq:first-target})-Eq.~(\ref{eq:target-propagation})) and update the forward weights (Eq.~(\ref{eq:def-ff-loss})), they acquire a very different meaning with our novel feedback weights training scheme. If we assume that an architecture $\mathcal{F}$ satisfies the JMC upon applying Alg.~(\ref{alg:L-DRL}) with fixed feedforward weights, then combining Eq.~(\ref{eq:def-delta-DTP}) and Eq.~(\ref{eq:jmc}) yields:
\begin{equation}
    \delta^n_{\rm GN} =- \left[\prod_{k=n}^{N-1} \partial_{s^{k}}F^{k^\top}\right]\cdot \frac{\partial \mathcal{L}_{\rm pred}}{\partial s^N} = \delta^n_{\rm BP},
\end{equation}
where $\delta^n_{\rm BP}$ denotes the activation updates computed by BP. Subsequently, given that the feedforward loss $\mathcal{L}_\theta^n$ defined in Eq.~(\ref{eq:def-ff-loss}) is updated by gradient descent, the whole DTP gradient computing scheme exactly implements BP rather than a hybrid between gradient descent and Gauss-Newton optimization. We now formally state our result.

\begin{theorem}[Gradient Matching Property]
\label{thm:feedforward-loss}
Let a feedforward architecture $\mathcal{F}$ defined per \cref{def:archi} which satisfies the JMC. Then the following holds:
\begin{equation}
    \forall n \in [1, N], \qquad \frac{\partial \mathcal{L}_{\rm pred}}{\partial \theta^n} = \lim_{\beta \to 0} \frac{1}{2\beta}\frac{\partial }{\partial \theta^n}\|t^n_\beta - s^n\|^2,
    \label{eq:thm-dtp-1}
\end{equation}
where the targets $(t^n_\beta )_{n\geq 1}$ obey the following recursive equations, $\forall n = N - 1 \dots 1$:
\begin{align}
\left\{
\begin{array}{l}
t^N_\beta = s^N -\beta\frac{\partial \mathcal{L}_{\rm pred}}{\partial s} \\
t^{n}_\beta = s^{n} + G^n(t_\beta^{n + 1};\omega^n) - G^n(s^{n + 1};\omega^n)
\label{eq:rec-eq-2}
\end{array}
\right.
\end{align}
\end{theorem}

\section{Experiments}
In this section, we present several experimental results supporting the above theory. We first numerically demonstrate the claims stated by Theorem~\ref{thm:feedback-loss} and Theorem~\ref{thm:feedforward-loss}, thereby showing the efficiency of the proposed approach to align feedforward and feedback weights (JMC) and subsequently compute DTP feedforward weight updates well aligned with BP gradients (GMP). Next, we present training simulation results on MNIST, F-MNIST and CIFAR-10, where our approach significantly outperforms \citet{meulemans2020theoretical}'s DTP. Finally, we report the best results ever obtained on ImageNet 32$\times$32 by a DTP algorithm. 

\subsection{Demonstrating the JMC}
\label{subsec:jmc}

\paragraph{Experimental set-up.} The goal of the following experiment is to compare \citet{meulemans2020theoretical}'s  DRL algorithm with our L-DRL approach in terms of their ability to align the (transposed) feedforward weights and their associated feedback weights for the last fully connected layer, and thereby realize the JMC in the output layer. We perform this test with randomly initialized and \emph{fixed} feedforward weights and on a \emph{single} randomly selected input batch $x$ (for a given seed). The choice of focusing only on the output layer is justified below. 

\paragraph{Architecture.} We consider a random batch of CIFAR-10 data along with a LeNet \cite{lecun1989backpropagation} architecture consisting of two convolutional layers and two fully connected (FC) layers. For both algorithms, we use the same feedforward pathway for the model. However since regular DRL prescribes by construction direct connections in the feedback pathway, the form of the $G^n$ functions used depends on the feedback algorithm used. \emph{For DRL}, we use the DDTP-linear architecture as per \citet{meulemans2020theoretical}, where the output layer are directly connected to each upstream layer via linear connections. Therefore, the parameters of the resulting $G^n$ functions have dimension ${\rm dim}(\omega^n) = s^n \times s^N$ for $n = N-1, \cdots, 1$. However, since the associated feedforward parameters $\theta^n$ have dimension ${\rm dim} (\theta^n) = s^{n+1} \times s^n$, we can only readily compare $\theta^{N^\top}$ and $\omega^N$ in the last FC layer. \emph{For L-DRL}, we use layer-wise $G^n$ functions such that ${\rm dim}(\omega^n) = s^n \times s^{n + 1}$ for $n = N-1, \cdots, 1$.
Full architecture details are included in the Appendix. 

\paragraph{Results.}
We illustrate  in Fig.~\ref{fig:thm-4.2} the results obtained. We show the angle (in degrees) and the relative distance between the last layer feedforward ($\theta^{N^\top}$) and feedback weights ($\omega^N$) throughout pure feedback training on a single input batch. Therefore, each feedback training iteration here corresponds to a feedback weight update on the \emph{same} input batch (for a given seed). However, we do use different input batches across different seeds. For each algorithm, the amount of noise and learning rates have been carefully tuned to achieve the minimal angles and distances after 5000 iterations, which we empirically found to be large enough to reach convergence for both algorithms. We observe that L-DRL achieves an angle of $\approx 3^\circ$ and a relative distance of $\approx 0$, while DRL can only reduce these quantities to $18.7^\circ$  and $\approx 1.8$ respectively. These results confirm that our L-DRL is more suited than DRL to achieve the JMC in the output layer. 

\begin{figure}[ht!]
  \begin{center}
    \includegraphics[width=0.50\textwidth]{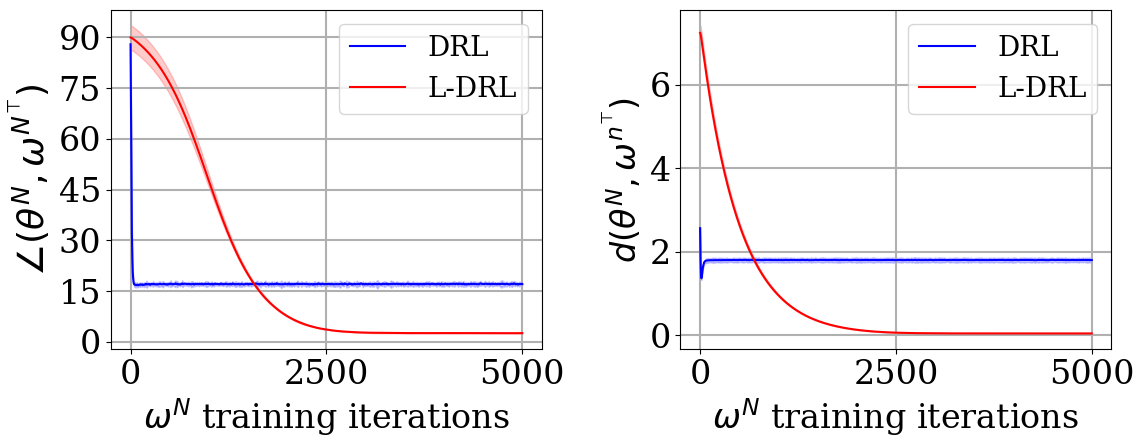}
  \end{center}
  \caption{Angle in degrees ($\angle(\theta^N, \omega^{N^\top})$) and relative distance ($d(\theta^N, \omega^{N^\top})$) between $\theta^N$ and $\omega^{N^\top}$  throughout feedback weight learning with L-DRL (ours) and DRL \cite{meulemans2020theoretical} with \emph{fixed} feedforward weights. 
  }
    \label{fig:thm-4.2}
\end{figure}

\subsection{Demonstrating the GMP}
\label{subsec:GMP}

\paragraph{Experimental set-up.} In this experiment, we want to demonstrate the ability of our proposed DTP to compute \emph{feedforward} weight updates closely matching those prescribed by BP (therefore achieving the GMP), assuming that the JMC is initially satisfied, as hypothesized by Theorem~\ref{thm:feedforward-loss}. Again here, we assume a \emph{single} randomly selected input batch $x$ (also with different input across different seeds). In contrast with the previous experiment though, we carry out this analysis across \emph{all} the layers. Indeed, regardless of the form of the $G^n$ functions (whether we use direct connections or not), the DTP feedforward weight updates can always be compared against those of BP. Given randomly sampled feedforward parameters $\theta^n$, we study five different feedback weight initialization schemes and associated targets computation: (a) $\omega^n$ are random and targets are computed through the DDTP-linear feedback pathway (${\rm DRL}_{\rm random}$); (b) same as (a) with targets computed through the layer-wise feedback pathway (L-${\rm DRL}_{\rm random}$); (c) $\omega^n$ are trained with DRL (DRL); (d) $\omega^n$ are trained with L-DRL (L-DRL); (e) Finally, $\omega^n = \theta^{n^\top}$ with targets propagated through the layer-wise feedback pathway (L-${\rm DRL}_{\rm sym}$).
For each of these situations, the feedforward DTP weight updates are thereafter obtained with Eq.~(\ref{eq:def-ff-loss}) on the one hand. On the other hand, we compute BP gradients via standard BP through the feedforward pathway. 

\paragraph{Architecture.} The architecture used for this experiment is the same LeNet architecture than the one used for the previous experiment, with two convolutional layers and two fully connected layers. 

\paragraph{Results.} We show on Fig.~\ref{fig:thm-4.3} the results obtained. The blue, red, green and purple bars correspond to the angle between DTP feedforward weight updates and those of BP ($\angle(\Delta \theta^n_{\rm DTP},\Delta \theta^n_{\rm BP})$) for the first Conv, second Conv, first FC and second FC layers respectively: the lower $\angle(\Delta \theta^n_{\rm DTP},\Delta \theta^n_{\rm BP})$, the more the GMP is satisfied. We show these quantities for each of the five feedback weight initialization mentioned above. We observe that upon training the feedback weights with L-DRL (compared to a random configuration), the GMP is significantly better satisfied ($\angle(\Delta \theta^n_{\rm DTP},\Delta \theta^n_{\rm BP})$ going from $\approx 90^\circ$ to $\lessapprox35^\circ$) than when trained with DRL ($\lessapprox 79^\circ$)), and almost as well as in the ideal situation with symmetrically initialized weights. Overall, these results confirm the prediction of Theorem~\ref{thm:feedforward-loss}.

\begin{figure}[ht!]
  \begin{center}
    \includegraphics[width=0.46\textwidth]{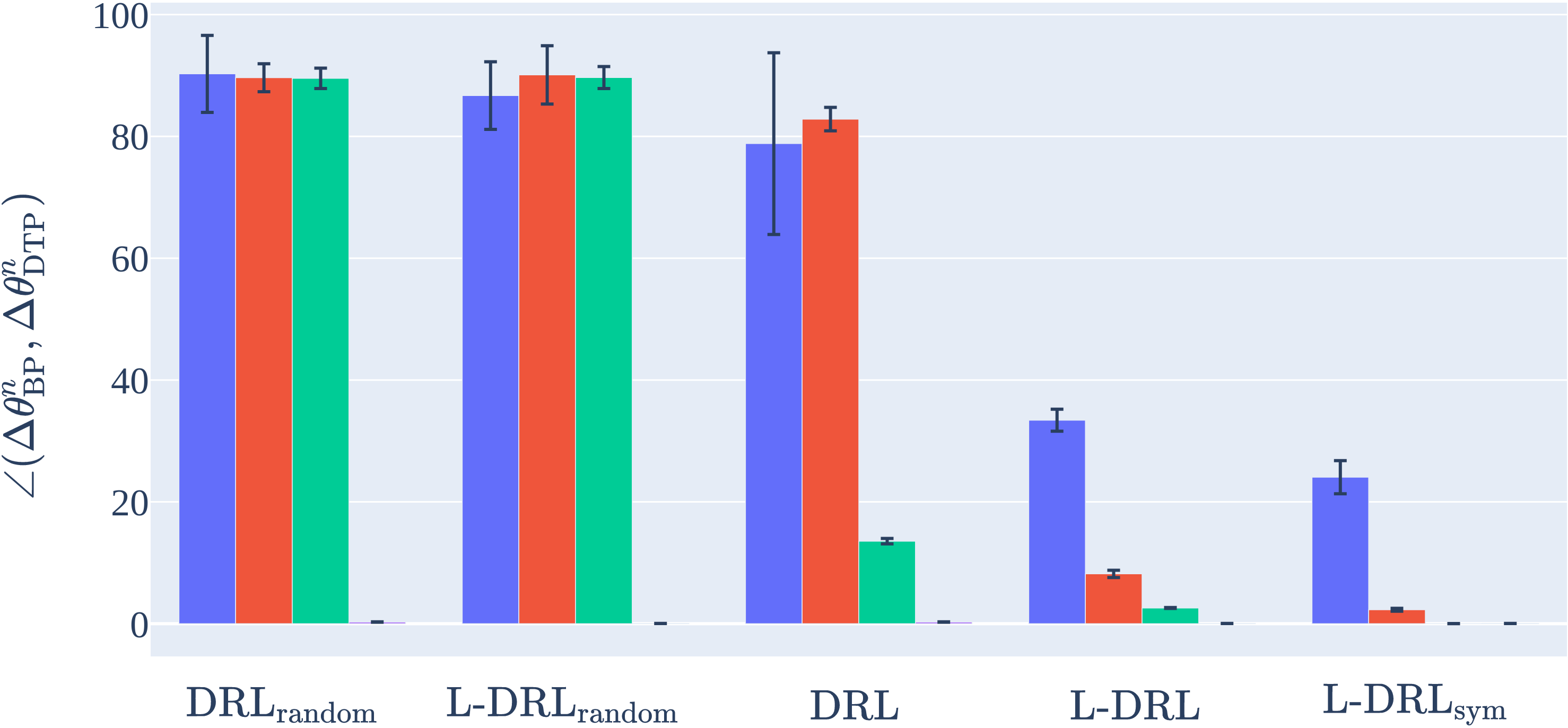}
  \end{center}
  \caption{Angle between the forward weight updates obtained through L-DRL (ours) or DRL \cite{meulemans2020theoretical} and those obtained through BP, for each layer, under various initial conditions.
  }
    \label{fig:thm-4.3}
\end{figure}

\subsection{DTP learning dynamics}
\label{subsec:train-experiments}

\paragraph{Experimental set-up.} We present here our training experiments obtained on MNIST, F-MNIST and CIFAR-10 with our implementation of DTP (refered to as ``DTP'' or ``Ours'' below) and that of \citet{meulemans2020theoretical} which we will refer to as ``DDTP''. While the previous DRL/L-DRL terminology concerns feedback weights training specifically, the term ``DDTP'' is used here to refer to the resulting feedforward weights training algorithm when GN targets are being computed, rather than the architecture itself. Also, we want to emphasize that two features of DDTP fundamentally differs from our DTP. First, our DTP is made to emulate BP while DDTP is a hybrid between GN optimization and BP as highlighted previously. Second, while DDTP employs feedback weights pre-training and subsequent interleaved epochs of pure feedback weights training, our DTP trains together \emph{at all times} feedforward weights and feedback weights and allowing for \emph{multiple} feedback weight updates per mini-batch. To disentangle these two aspects and ensure a fair comparison between our DTP and DDTP, we propose two different implementations of DDTP. 

\emph{Simple} DDTP (``s-DDTP'') is the standard DDTP implementation of \citet{meulemans2020theoretical} that yields their best training results. For s-DDTP, training starts with $N_{\omega, \rm i}$ epochs of pure feedback weights training, then at each subsequent epoch feedback weights and feedforward weights are both updated once per batch, and each of these epoch is followed by $N_{\omega, \rm e}$ epoch of pure feedback training. Therefore, denoting $N_{\rm \theta}$ the number of epochs where the feedforward weights are trained, 
there are $N_{\rm \omega} = N_{\rm \omega, i} + N_{\theta}\times (1 + N_{\omega, \rm e})$ epochs where the feedback weights are trained, therefore $\mathcal{O}(N_{\rm \omega, i} + N_{\theta}\times (1 + N_{\omega, \rm e}))$ feedback weight updates.

We define \emph{Parallel} DDTP (``p-DDTP'') as a variant of DDTP where there is no initial feedback pre-training ($N_{\omega, \rm i} = 0$), nor interleaved epochs of pure feedback training ($N_{\omega, \rm e} = 0$), but where feedback weights and feedforward weights are always trained altogether, with $K$ feedback weight updates per batch, yielding $\mathcal{O}(N_{\theta} \times K)$ feedback weight updates. Therefore, p-DDTP has the same complexity cost for feedback weights training as in our DTP. We use the same architecture in this study as in Section ~\ref{subsec:jmc}-\ref{subsec:GMP}.

\begin{figure*}[ht!]
  \begin{center}
    \includegraphics[width=0.9\textwidth]{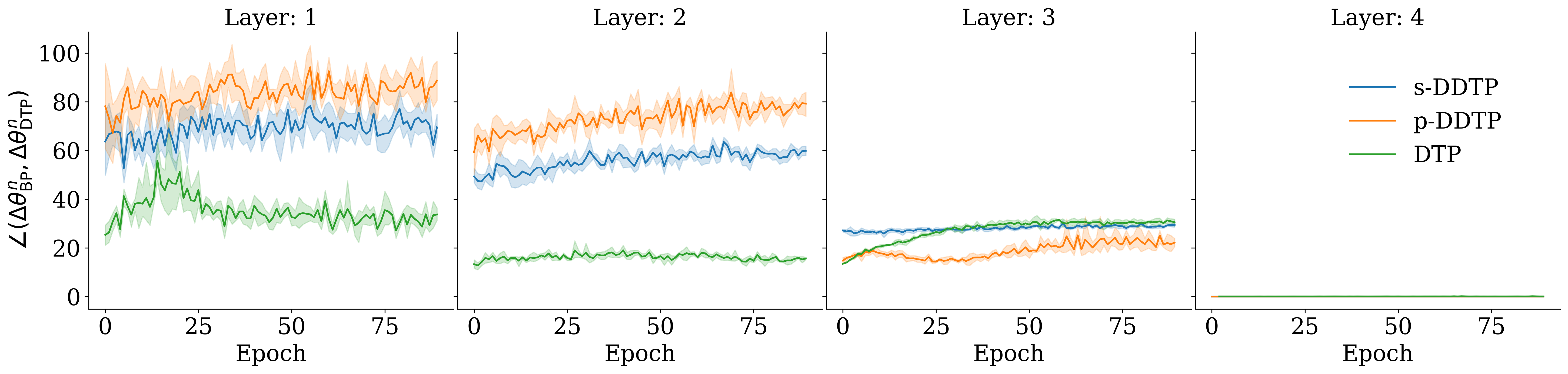}
  \end{center}
  \caption{Angle between forward weight updates obtained through DTP (ours), s-DDTP, p-DDTP and those obtained through back-propogation for each layer throughout training on CIFAR-10 with LeNet architecture. Each epoch represents a feedforward training epoch, pure feedback training epochs for s-DDTP are not displayed. 
  }
    \label{fig:bpweightupdatesangle}
\end{figure*}

\paragraph{Results.} We display in Table~\ref{tab:res-lenet} the  accuracies obtained with our DTP, s-DDTP and p-DDTP on MNIST, Fashion MNIST (``F-MNIST'') and CIFAR-10. Our DTP outperforms s-DDTP and p-DDTP on all tasks, by $\approx 0.3 \%$
on MNIST and F-MNIST, by at least $\approx 9\%$ on CIFAR-10 and is within $\approx 1\%$ of the BP baseline performance. While p-DDTP slightly outperforms s-DDTP on MNIST and F-MNIST, it performs worse than s-DDTP on CIFAR-10, suggesting that DDTP does not benefit much from multiple feedback weight updates per batch. An important conclusion to be drawn here is that the gap in performance between our DTP and DDTP is not due to updating the feedback weights multiple times per training batch, but more fundamentally to the feedback training scheme at use (L-DRL for our DTP, or DRL for DDTP), yielding better feedforward error signals with our DTP. This conclusion is also confirmed by Fig.~\ref{fig:bpweightupdatesangle} where we plot the angle between the DTP feedforward weight updates and those of BP ($\angle(\Delta \theta^n_{\rm BP}, \Delta \theta^n_{\rm DTP})$) throughout learning CIFAR-10, for each layer and each algorithm. While the angles obtained by DTP, s-DDTP and p-DDTP are comparable for the last two (FC) layers ($\approx 0^\circ$), they are at least twice as smaller for DTP compared to s-DDTP and p-DDTP in the first two (Conv) layers. Finally, these angles are $\approx 15^\circ$ smaller for s-DDTP compared to p-DDTP. Consequently, these curves directly account for the discrepancies in results on CIFAR-10 reported in Table~\ref{tab:res-lenet}.

\begin{table}[ht]
\caption{Accuracies ($\%$) obtained with BP, DDTP and our DTP on the LeNet architecture for MNIST, F-MNIST and CIFAR-10 test set. Each result is in terms of the mean and standard deviation obtained over five different seeds.}
\label{tab:res-lenet}
\vskip 0.15in
\begin{center}
\begin{small}
\begin{sc}
\begin{tabular}{lccc}
\toprule
{} & MNIST & F-MNIST & CIFAR-10  \\
\midrule

s-DDTP &$98.59_{\pm 0.16} $  & $88.86_{\pm 0.44} $ & $76.33_{\pm 0.27} $\\
p-DDTP & $98.58_{\pm 0.13} $ &  $89.42_{\pm 0.69} $ & $72.15_{\pm 0.29} $\\
{\bfseries Ours} &  $ \mathbf{98.93_{\pm 0.04}}$ & $\mathbf{90.91_{\pm 0.17}}$ & $\mathbf{85.33_{\pm 0.32}}$ \\
\midrule
BP  & $98.92_{\pm 0.04} $ & $91.94_{\pm 0.33}$ & $86.34_{\pm 0.35} $ \\
\bottomrule
\end{tabular}
\end{sc}
\end{small}
\end{center}
\vskip -0.1in
\end{table}

\subsection{Towards scaling up DTP}
Since we have demonstrated that our DTP learns better error signals to update the feedforward weights than DDTP with consistently better performance, we focus in this section on learning a slightly deeper and wider architecture on CIFAR-10 and ImageNet 32$\times$32 \cite{NIPS2016_b1301141}, a downsampled version of the full ImageNet data. For these experiments, we employ a 6-layers VGG-like architecture, consisting of 5 Conv layers and 1 FC layer (see Appendix for architecture details). 
\paragraph{Results.} We report our results in Table~\ref{tab:res-simplevgg}. With this choice of architecture, our DTP achieves $89.38\%$ accuracy on CIFAR-10 and $60.54\%$ top-5 accuracy on ImageNet 32 $\times$ 32, which is both cases within $< 1\%$ of the BP baseline.
\begin{table}[ht!]
\caption{Accuracies ($\%$) obtained on CIFAR-10 with BP and our DTP on a VGG-like architecture. Each result is in terms of the mean and standard deviation obtained over five different seeds. We also report below the current best CIFAR-10 accuracies obtained by DTP in the literature on any architecture.}
\label{tab:res-simplevgg}
\vskip 0.15in
\begin{center}
\begin{small}
\begin{sc}
\begin{tabular}{cc}
\toprule
& Accuracy \\
\midrule
BP & $89.07_{\pm 0.22} $ \\
{\bfseries Ours} & $\mathbf{89.38_{\pm 0.20 }}$ \\  
\hline
\citet{meulemans2020theoretical}& $76.01$ \\
\citet{bartunov2018assessing} & $60.53$ \\
\bottomrule
\end{tabular}
\end{sc}
\end{small}
\end{center}
\vskip -0.1in
\end{table}
\begin{table}[ht!]
\caption{Top-1 and Top-5 Accuracies for ImageNet 32$\times$32 validation set obtained with BP and our DTP on a VGG-like architecture on a single seed.}
\label{tab:res-2}
\vskip 0.15in
\begin{center}
\begin{small}
\begin{sc}
\begin{tabular}{lcc}
\toprule
 & Top-1 & Top-5 \\
\midrule
BP& $37.47 $ & $61.25$ \\
Ours & $36.81$ & $60.54$ \\
\bottomrule
\end{tabular}
\end{sc}
\end{small}
\end{center}
\vskip -0.1in
\end{table}

\section{Discussion}\vspace{-5pt}
Training feedforward weights with Gauss-Newton targets results in optimal updates to move the feedforward activations towards their associated target, yet they appear sub-optimal to decrease the prediction loss \cite{meulemans2020theoretical}, which calls for the design of a principled way to build backprop-like targets. In this work, we have demonstrated the benefits of such an approach, with mathematically and experimentally grounded arguments. We showed the efficiency of our L-DRL algorithm to align feedforward and (transposed) feedback weights  and therefore achieve the Jacobian matching condition (JMC). We also showed that the resulting feedforward weight updates prescribed by Difference Target Propagation (DTP) closely match those of BP, a property we called the gradient-matching property (GMP). Our DTP implementation subsequently outperforms DDTP \citep{meulemans2020theoretical} on all training tasks and approaches the BP baseline performance. We also consistently showed that the more the GMP is satisfied throughout learning, the better the resulting performance. The best CIFAR-10 performance obtained by our DTP is $\approx 13 \%$ higher than the existing DTP performances reported in the literature \cite{bartunov2018assessing, meulemans2020theoretical} and to our knowledge this is the first report of a DTP performance closely matching that of BP on such a complex task as ImageNet 32$\times$32.

{\bf Limitations and Future Work.} Our prescription to run several feedback weight updates per training batch entails longer simulation times but may be biologically plausible since local recurrent paths will have shorter axons that should have much shorter delays than long-range paths with complex, long axons \cite{debanne2004information, debanne2011axon}.
Future work could be done to leverage the parallelism allowed by our layer-wise feedback weight training strategy to accelerate training and subsequently scale up our DTP implementation to ImageNet on deeper architectures.

Our code is available at {\small\url{https://github.com/BPTargetDTP/ScalableDTP}}.

\paragraph{Acknowledgements.} BR was supported by NSERC (Discovery Grant: RGPIN-2020-05105; Discovery Accelerator Supplement: RGPAS-2020-00031) and CIFAR (Canada CIFAR AI Chair and Learning in Machines and Brains Program). EB and AM are supported by NSERC Discovery Grant RGPIN-2021-04104. We acknowledge resources provided by Compute Canada. YB was funded by NSERC and CIFAR.

\bibliography{example_paper}
\bibliographystyle{icml2022}

\newpage
\appendix
\onecolumn

\section{Theoretical results}
\subsection{Feedback weights training}
We re-state our main theorem for feedback weights training (Theorem~(\ref{thm:feedback-loss}) in the main).

\begin{theorem}
\label{thm:feedback-loss-appendix}
Let: 
\begin{align}
    \widehat{\mathcal{L}}_\omega^n &= -\frac{1}{\sigma^2}\epsilon^\top \cdot \left( r^n_\epsilon - s^n\right)+ \frac{1}{2 \sigma^2}\left\| r_\eta^n - s^n\right\|^2, \\
    \mathcal{L}_\omega^n &= \frac{1}{2}\left \|\partial_{s^n}F^{n^\top} - \partial_{s^{n+1}}G^{n+1} \right\|^2_F.
\end{align}

Then:

\begin{align}
\lim_{\sigma \to 0}\mathbb{E}_{\epsilon, \eta}\left[ \hat{\mathcal{L}}_\omega^n\right] &= -\left \langle \partial_{s^n} F^{n^\top}, \partial_{s^{n+1}}G^n\right\rangle_F + \frac{1}{2} \left\| \partial_{s^{n+1}}G^n \right \|_F^2
\label{eq:thm-fb-res-1}\\
\frac{\partial}{\partial \omega^n}\lim_{\sigma \to 0}\mathbb{E}_{\epsilon, \eta}\left[ \hat{\mathcal{L}}_\omega^n\right] &= \frac{\partial \mathcal{L}_\omega^n}{\partial \omega^n},
\label{eq:thm-fb-res-2}
\end{align}
\end{theorem}

\begin{proof}

We have:

\begin{align}
 \epsilon^\top \cdot \left( r^n_\epsilon - s^n\right) &=
 \epsilon^\top \cdot \left(G^n \circ F^n (s^n + \epsilon)  - G^n \circ F^n (s^n)  \right) \nonumber \\
 &= \epsilon^\top \cdot \left( \partial_{s^{n+1}}G^n \cdot \partial_{s^n}F^n \cdot \epsilon + o(\|\epsilon\|)   \right) \nonumber\\
&= {\rm Tr}(\epsilon\cdot \epsilon^\top \cdot \partial_{s^{n+1}}G^n \cdot \partial_{s^n}F^n) + o(\|\epsilon\|^2).
\label{demo:feedback-step1}
\end{align} \\

Likewise:

\begin{equation}
\left\| r_\eta^n - s^n\right\|^2 = {\rm Tr}\left(\eta\cdot \eta^\top \cdot \partial_{s^{n+1}}G^n \cdot \partial_{s^{n+1}}G^{n^\top} \right) + o(\| \eta\|^2)\label{demo:feedback-step2}.
\end{equation}

Moreover, since $\epsilon \sim \mathcal{N}\left(0, \sigma^2\right)$ and $\eta \sim \mathcal{N}\left(0, \sigma^2\right)$, $\mathbb{E}_\epsilon \left [ o(\|\epsilon\|^2) \right] = \mathbb{E}_\eta \left [ o(\|\eta\|^2) \right]= o(\sigma^2)$. Altogether with Eq.~(\ref{demo:feedback-step1}) and Eq.~(\ref{demo:feedback-step2}), we finally obtain:

\begin{align}
\mathbb{E}_{\epsilon, \eta}\left[ \hat{\mathcal{L}}_\omega^n\right]
&= - \frac{1}{\sigma^2}{\rm Tr}(\underbrace{\mathbb{E}_\epsilon\left[\epsilon\cdot \epsilon^\top\right]}_{=\sigma^2 \times \mathbf{1}} \cdot \partial_{s^{n+1}}G^n \cdot \partial_{s^n}F^n) + \frac{1}{2 \sigma^2}
+ {\rm Tr}\left(\underbrace{\mathbb{E}_{\eta}\left[\eta\cdot \eta^\top\right]}_{=\sigma^2 \times \mathbf{1}} \cdot \partial_{s^{n+1}}G^n \cdot \partial_{s^{n+1}}G^{n^\top} \right) + o(1) \nonumber \\
&= -\left \langle \partial_{s^n} F^{n^\top}, \partial_{s^{n+1}}G^n\right\rangle_F + \frac{1}{2} \left\| \partial_{s^{n+1}}G^n \right \|_F^2 + o(1).
\end{align}

Therefore, sending $\sigma \to 0$ yields the desired result Eq.~(\ref{eq:thm-fb-res-1}). Finally, noticing that:

\begin{equation}
    \lim_{\sigma \to 0}\mathbb{E}_{\epsilon, \eta}\left[ \hat{\mathcal{L}}_\omega^n\right] = \mathcal{L}_\omega^n - \frac{1}{2}\left\|\partial_{s^n}F^n\right\|_F^2,
    \label{demo:feedback-step3}
\end{equation}

with the second term of Eq.~(\ref{demo:feedback-step3}) not depending on the feedback weights $\omega^n$, Eq.~(\ref{eq:thm-fb-res-2}) is straightforward. 
\end{proof}

\subsection{Feedforward weights training}
We re-state here our main theorem for feedforward weights training (Theorem~\ref{thm:feedforward-loss}) in the main). 

\begin{theorem}[Gradient Matching Property]
\label{thm:feedforward-loss-2}
Let a feedforward architecture $\mathcal{F}$ defined per \cref{def:archi} which satisfies the JMC. Then the following holds:

\begin{equation}
    \forall n \in [1, N], \qquad \frac{\partial \mathcal{L}_{\rm pred}}{\partial \theta^n} = \lim_{\beta \to 0} \frac{1}{2\beta}\frac{\partial }{\partial \theta^n}\|t^n_\beta - s^n\|^2,
    \label{eq:thm-dtp-1-bis}
\end{equation}

where the targets $(t^n_\beta )_{n\geq 1}$ obey the following recursive equations, $\forall n = N \dots 2$:

\begin{align}
\left\{
\begin{array}{l}
t^N_\beta = s^n -\beta\frac{\partial \mathcal{L}_{\rm pred}}{\partial s} \\
t^{n}_\beta = s^{n} + G^n(t_\beta^{n + 1};\omega^n) - G^n(s^{n + 1};\omega^n)
\label{demo:feedforward-loss-step1}
\end{array}
\right.
\end{align}
\end{theorem}

\begin{proof}
First, note we have:

\begin{align}
\left\{
\begin{array}{l}
t^N_\beta - s^n =  -\beta\frac{\partial \mathcal{L}_{\rm pred}}{\partial s}, \\
t^{n}_\beta - s^{n} = \partial_{s^{n+1}}G^n \cdot \left(t^{n+1}_\beta - s^{n+1}_\beta \right) + o\left(\left\|t^{n+1}_\beta - s^{n+1}_\beta \right\|\right).
\label{demo:feedforward-loss-step2}
\end{array}
\right.
\end{align}

Since $t^N_\beta - s^N = o(\beta)$, by immediate induction $t^n_\beta - s^n = o(\beta) \quad \forall n=N-1 \cdots 1$. 
Denoting $\hat{\delta}^n_{\rm DTP}(\beta) \triangleq \frac{t^n - s^n}{\beta}$, we therefore obtain:

\begin{align}
\left\{
\begin{array}{l}
\hat{\delta}^N_{\rm DTP}(\beta) =  -\frac{\partial \mathcal{L}_{\rm pred}}{\partial s} +o(1), \\
\hat{\delta}^{n}_{\rm DTP}(\beta) = \partial_{s^{n+1}}G^n \cdot \hat{\delta}^{n+1}_{\rm DTP}(\beta)+ o\left(1\right).
\label{demo:feedforward-loss-step3}
\end{array}
\right.
\end{align}

Furthermore note that, treating $t^n_\beta$ as a constant, we also have:

\begin{align}
\frac{1}{2\beta}\frac{\partial }{\partial \theta^n}\|t^n_\beta - s^n\|^2 &= -\frac{1}{\beta}\partial_{\theta^n}F^{n^\top}\cdot (t^n_\beta - s^n) \nonumber \\
& = -\partial_{\theta^n}F^{n^\top}\cdot \hat{\delta}^n_{\rm DTP}(\beta) \label{demo:feedforward-loss-step4}.
\end{align}

Finally, sending $\beta \to 0$, defining:

\begin{align}
   \delta^n_{\rm DTP} &\triangleq \lim_{\beta \to 0} \hat{\delta}^n(\beta) \\
   \Delta^n_{\theta, \rm DTP} &\triangleq \lim_{\beta \to 0}\frac{1}{2\beta}\frac{\partial }{\partial \theta^n}\|t^n_\beta - s^n\|^2
\end{align}

along with the JMC property $\partial_{s^{n+1}}G^n = \partial_{s^n} F^{n^\top}$, we obtain:

\begin{align}
\left\{
\begin{array}{l}
\delta^N_{\rm DTP} =  -\frac{\partial \mathcal{L}_{\rm pred}}{\partial s}, \\
\delta^{n}_{\rm DTP} = \partial_{s^{n+1}} F^{n^\top} \cdot \hat{\delta}^{n+1}_{\rm DTP} \\
\Delta_{\theta, \rm DTP}^n = -\partial_{\theta^n}F^{n^\top}\cdot \hat{\delta}^n_{\rm DTP}
\label{demo:feedforward-loss-step5}
\end{array}
\right.
\end{align}

Note that Eq.~(\ref{demo:feedforward-loss-step5}) is equivalent to computing $\frac{\partial \mathcal{L}_{\rm pred}}{\partial \theta^n}$ by backprop, yielding the desired result Eq.~(\ref{eq:thm-dtp-1-bis}).
\end{proof}

\section{A concrete example with explicit equations}
We detail for completeness and clarity all the equations for the neural dynamics and the learning rules of the forward and of the backward weights for a LeNet-like architecture with two Conv layers and one fully connected layer for the sake of simplicity.

\paragraph{Forward operations.}
\begin{align*}
s^1 &= F^0(x;\theta^0) = \sigma(\theta^0 \star x),\\
s^2 &= F^1(s^1;\theta^1) = \sigma(\theta^1 \star s_1), \\
s^3 &= F^2(s_2;\theta^2) = \theta^2 \cdot s^2,\\
\hat{y} &= {\rm softmax}(s_3),
\end{align*}

where we implicitly assume the flattening operation between $s^2$ and $s^3$ for notational convenience. 

\paragraph{Backward operations.}
We assume here the following feedback operators $G^1$ and $G^2$ associated with $F^1$ and $F^2$ respectively: 

\begin{align*}
    G^2(s^3;\omega^2) &= \omega^2 \cdot s^3 \\
    G^1(s^2;\omega^1) &= \omega^1\star\sigma(s^2).
\end{align*}

Again, note that there is no $G^0$ feedback operator paired to $F^0$ since we do not need to propagate error signals down to the input layer.

\paragraph{Feedback weights training.} Given input noises $\epsilon^2$ and $\eta^3$ in layers $s^2$ and $s^3$ respectively, we update $\omega^2$ so as to minimize the loss $\mathcal{L}_{\omega^2}$. More precisely, $\epsilon^2$, $\eta^3$ and $\mathcal{L}_{\omega^2}$ are defined as:

\begin{align*}
    \epsilon^2 &\sim \mathcal{N}(0, \sigma^2), \\
    \eta^3 &\sim \mathcal{N}(0, \sigma^2), \\
    \mathcal{L}_{\omega^2} &= - \left(\epsilon^{2}\right)^\top \cdot (G^2(F^2 (s^2 + \epsilon^2)) - G^2(s^3)) + \frac{1}{2}\left\|G^2(s^3 + \eta^3) - G^2(s^3)\right\|^2,
\end{align*}

which results in the weight update for $\omega^2$:

\begin{equation}
\Delta \omega^2 = \epsilon^2 \cdot(\theta^2 \cdot \Delta x^2)^\top - (\omega^2 \cdot \eta^3) \cdot \Delta y^{3^\top}.
\label{eq:lr-omega3}
\end{equation}

Similarly, we train the feedback convolutional filters $\omega^1$ by injecting the input noise $\epsilon^1$ and $\eta^2$ in $s^1$ and $s^2$ respectively and minimizing the loss $\mathcal{L}_{\omega^1}$ defined as:  

\begin{align*}
    \epsilon^1 &\sim \mathcal{N}(0, \sigma^2), \\
    \eta^2 &\sim \mathcal{N}(0, \sigma^2), \\
    \mathcal{L}_{\omega^1} &= - \left(\epsilon^{1}\right)^\top \cdot (G^1(F^1 (s^1 + \epsilon^1)) - G^1(s^2)) + \frac{1}{2}\left\|G^1(s^2 + \eta^2) - G^1(s^2)\right\|^2,
\end{align*}

which results in the weight update for the convolutional feedback filters $\omega^1$:         

\begin{equation}
\Delta \omega^1 = \epsilon^1 \star (\sigma(F^1(s^1 + \epsilon^1)) - \sigma(s^2)) - (G^1(s^2 + \eta^2) - G^1(s^2))\star(\sigma(s^2 + \eta^2) - \sigma(s^2)).
\label{eq:lr-omega2}
\end{equation}

\paragraph{Forward weights training.}

We compute the first target $t^3_\beta$ and associated weight update $\Delta \theta^2$ as:

\begin{align}
    t^3_\beta &= s^3 + \beta(\hat{y} - y), \\
    \Delta \theta^2 &= \frac{1}{\beta}(t^3_\beta - s^3)\cdot s^{2^\top}.
\end{align}

Then, the target $t^3_\beta$ passes through $G^2$, yielding the target $t^2_\beta$ and associated weight update $\Delta \theta^1$: 

\begin{align}
    t^2_\beta &= s^2 + G^2(t^3;\omega^2) - G^2(s^3;\omega^2),\\
    \Delta \theta^1 &= \frac{1}{\beta}((t^2_\beta - s^2)\odot \sigma'(s^2))\star s^{1}.
    \label{eq:learning-rule-theta2}
\end{align}

Similarly, we compute $t^1_\beta$ and $\Delta \theta^0$ as:

\begin{align}
    t^1_\beta &= s^1 + G^1(t^2;\omega_2) - G^1(s^2;\omega_2)\\
    \Delta \theta^0 &= \frac{1}{\beta}((t^1_\beta - s^1)\odot\sigma'(s^1))\star x
    \label{eq:learning-rule-theta1}    
\end{align}

\section{Architecture Details}
Table ~\ref{table:architectures} we give the details of the two representative architectures studied in our work.
\begin{table}[h!]
\centering
\begin{center}
\begin{tabular}{ c c c c }
 LeNet & VGGNet\\ \hline
 Conv 5x5x32 (stride=1, pad=2) & Conv 3x3x128 (stride=1, pad=1)\\  
 Maxpool 3x3 (stride=2, pad=1) & Maxpool 2x2 (stride=2, pad=0)\\
 Conv 5x5x64 (stride=1, pad=2) & Conv 3x3x128 (stride=1, pad=1)\\
 Maxpool 3x3 (stride=2, pad=1) & Maxpool 2x2 (stride=2, pad=0)\\
 FC 512 & Conv 3x3x256 (stride=1, pad=1)\\
 FC+Softmax 10 & Maxpool 2x2 (stride=2, pad=0)\\
 - & Conv 3x3x256 (stride=1, pad=1)\\
 - & Maxpool 2x2 (stride=2, pad=0)\\
 - & Conv 3x3x512 (stride=1, pad=1)\\
 - & Maxpool 2x2 (stride=2, pad=0)\\
 - & FC+Softmax 10 \\
 \bottomrule
\end{tabular}
\end{center}
   \caption{Architectures described by layer}
  \label{table:architectures}
\end{table}

\section{Hyperparameters} 
In Tables~\ref{table:bphparamslenet},\ref{table:dtphparamslenet},\ref{table:sddtphparamslenet},\ref{table:pddtphparamslenet} we report the hyperparameters for each method and dataset studied. In both CIFAR-10 and Imagenet 32$\times$32 experiments we use the same data augmentation consisting of random horizontal flipping with 0.5 probability and random cropping with padding 4.

\begin{table}[h!]
\centering
  \begin{tabular}{p{4cm} p{2.5cm} p{2.5cm} p{2.5cm}}
    \toprule
    \multirow{2}{*}{Hyperparameter} & \multicolumn{3}{c}{Dataset} \\
      \multicolumn{1}{c}{ } &
      \multicolumn{1}{c}{ } & 
      \multicolumn{1}{c}{ } \\
      & {MNIST} & {F-MNIST} & {CIFAR-10} \\
      \midrule
 channels & [32, 64] & [32, 64] & [32, 64] \\ \hline
 activation & ELU & ELU & ELU \\ \hline
 $\text{lr}_{\text{f}}$ & 0.007938 & 0.01374 & 0.03 \\ \hline
 forward optimizer & SGD & SGD & SGD\\ \hline
 forward momentum & 0.9 & 0.9 & 0.9\\ \hline
 $\text{wd}_{\text{f}}$ & 0.0001 & 0.0001 & 0.0001\\ \hline
 scheduler & cosine & cosine & cosine\\ \hline
 scheduler eta min & 0.00001 & 0.00001 & 0.00001\\ \hline
 scheduler Tmax & 85 & 85& 85 \\ \hline
 scheduler interval/frequency & epoch/1 & epoch/1 & epoch/1\\ \hline
 initialization & kaiming uniform & kaiming uniform & kaiming uniform \\ \hline
 batch size & 166 & 140 & 193 \\ \hline
 epochs & 40 & 40 & 90 \\ 
    \bottomrule
  \end{tabular}
  \caption{Tuned BP LeNet hyperparameters}
  \label{table:bphparamslenet}
\end{table}

\begin{table}[h!]
\centering
  \begin{tabular}{p{4cm} p{4cm} p{4cm} p{4cm}}
    \toprule
    \multirow{2}{*}{Hyperparameter} & \multicolumn{3}{c}{Dataset} \\
      \multicolumn{1}{c}{ } &
      \multicolumn{1}{c}{ } & 
      \multicolumn{1}{c}{ } \\
      & {MNIST} & {F-MNIST} & {CIFAR-10} \\
      \midrule
 channels & [32, 64] & [32, 64] & [32, 64] \\ \hline
 $\beta$ & 0.4768550374762699  &  0.3651375179883248 & 0.46550286113514694\\ 
 \hline
 $\sigma$ & [0.4, 0.4, 0.2] & [0.3885862406080412, 0.2373096461112338, 0.15496346129996677] & [0.41640228838517584, 0.2826261146623929, 0.19953820693586016]\\
 \hline
 activation & ELU & ELU & ELU\\ \hline
 $\text{lr}_{\text{f}}$ & 0.02046745493369468 & 0.005697551532646145  & 0.0.01\\ \hline
 forward optimizer & SGD & SGD & SGD\\ \hline
 forward momentum & 0.9 & 0.9 & 0.9\\ \hline
 $\text{wd}_{\text{f}}$ & 0.0001 & 0.0001 & 0.0001\\ \hline
 $\text{lr}_{\text{fb}}$ & [0.06813589667087301, 0.006643595431387696, 0.018743666114857397] & [0.01099976940762419, 0.00026356477629680596, 0.06692513019217786] & [0.001,0.005,0.045]\\ \hline
 feedback training iterations & [18, 23, 12]  & [41, 15, 19] & [41, 51, 24]\\ \hline
 backward optimizer & SGD & SGD & SGD\\ \hline
 backward momentum & 0.9 & 0.9 & 0.9\\ \hline
 $\text{wd}_{\text{fb}}$ & None & None & None\\ \hline
 scheduler & cosine & cosine & cosine\\ \hline
 scheduler eta min & 0.00001 & 0.00001 & 0.00001\\ \hline
 scheduler Tmax & 85 & 85 & 85\\ \hline
 scheduler interval/frequency & epoch/1 & epoch/1 & epoch/1\\ \hline
 initialization & kaiming uniform & kaiming uniform & kaiming uniform \\ \hline
 batch size & 107 & 33 & 100\\ \hline
 epochs & 40 & 40 & 90\\ 
    \bottomrule
  \end{tabular}
  \caption{Tuned DTP LeNet hyperparameters}
  \label{table:dtphparamslenet}
\end{table}

\begin{table}[h!]
\centering
  \begin{tabular}{p{4cm} p{4cm} p{4cm}}
    \toprule
    \multirow{2}{*}{Hyperparameter} & \multicolumn{2}{c}{Dataset} \\
      \multicolumn{1}{c}{ } &
      \multicolumn{1}{c}{ } \\
      & {CIFAR-10} & {ImageNet 32$\times$32} \\
      \midrule
 channels & [128, 128, 256, 256, 512] & [128, 128, 256, 256, 512]\\ \hline
 $\beta$ & 0.7  &  0.7\\ 
 \hline
 $\sigma$ & [0.4, 0.4, 0.2, 0.2, 0.08] & [0.4, 0.4, 0.2, 0.2, 0.08]\\
 \hline
 activation & ELU & ELU\\ \hline
 $\text{lr}_{\text{f}}$ & 0.05 & 0.01\\ \hline
 forward optimizer & SGD & SGD\\ \hline
 forward momentum & 0.9 & 0.9\\ \hline
 $\text{wd}_{\text{f}}$ & 0.0001 & 0.0001\\ \hline
 $\text{lr}_{\text{fb}}$ & [1e-4, 3.5e-4, 8e-3, 8e-3, 0.18] & [1e-4, 3.5e-4, 8e-3, 8e-3, 0.18]\\ \hline
 feedback training iterations & [20, 30, 35, 55, 20]  & [25, 35, 40, 60, 25]\\ \hline
 backward optimizer & SGD & SGD\\ \hline
 backward momentum & 0.9 & 0.9\\ \hline
 $\text{wd}_{\text{fb}}$ & None & None\\ \hline
 scheduler & cosine & cosine\\ \hline
 scheduler eta min & 0.00001 & 0.00001\\ \hline
 scheduler Tmax & 85 & 85\\ \hline
 scheduler interval/frequency & epoch/1 & epoch/1\\ \hline
 initialization & kaiming uniform & kaiming uniform\\ \hline
 batch size & 128 & 256\\ \hline
 epochs & 90 & 90\\ 
    \bottomrule
  \end{tabular}
  \caption{Tuned DTP VGGNet hyperparameters}
  \label{table:dtphparamslenet}
\end{table}

\begin{table}
\centering
  \begin{tabular}{p{4cm} p{4cm} p{4cm} p{4cm}}
    \toprule
    \multirow{2}{*}{Hyperparameter} & \multicolumn{3}{c}{Dataset} \\
      \multicolumn{1}{c}{ } &
      \multicolumn{1}{c}{ } & 
      \multicolumn{1}{c}{ } \\
      & {MNIST} & {F-MNIST} & {CIFAR-10} \\
      \midrule
 channels & [32, 64] & [32, 64] & [32, 64] \\ \hline
 target stepsize $\eta$ & 0.04385 & 0.04385 & 0.015962099947441903 \\ \hline
 $\beta_1$ & 0.9 & 0.9 & 0.9\\ \hline
 $\beta_2$ & 0.999 & 0.999 & 0.999 \\ \hline
 $\epsilon$ & [6.533e-05, 1.175e-05, 6.843e-05, 2.564e-05] & [6.533e-05, 1.175e-05, 6.843e-05, 2.564e-05] & [2.7867895625009e-08, 1.9868935703787622e-08, 4.515242618159344e-06, 4.046144976139705e-05]\\ \hline
 $\epsilon_{\text{fb}}$ & 9.506e-08 & 9.506e-08 & 7.529093372180766e-07 \\ \hline
 $\beta_{1,\text{fb}}$ & 0.999 & 0.999 & 0.999 \\ \hline
 $\beta_{2,\text{fb}}$ & 0.999 & 0.999 & 0.999 \\ \hline
 $\sigma$ & [4.788e-05, 0.0008712, 0.0002377, 3.966e-05] & [4.788e-05, 0.0008712, 0.0002377, 3.966e-05] & [0.00921040366516759, 0.00921040366516759, 0.00921040366516759, 0.00921040366516759]\\ \hline
 activation & tanh & tanh & tanh\\ \hline
 $\text{lr}_{\text{f}}$ & [0.001694, 0.09782, 0.02479, 0.001937] & [0.001694, 0.09782, 0.02479, 0.001937] & [0.00025935571806476586, 0.000885500279951265, 0.0001423047695105589, 3.3871035558126015e-06] \\ \hline
 forward optimizer & Adam & Adam & Adam \\ \hline
$\text{wd}_{\text{f}}$ & 0 & 0 & 0 \\ \hline
$\text{lr}_{\text{fb}}$ &  0.0001614 & 0.0001614 & 0.0045157498494467095 \\ \hline
 feedback training iterations & [1, 1, 1, 1] & [1, 1, 1, 1] & [1, 1, 1, 1] \\ \hline
 feedback activation & linear & linear & linear\\ \hline
 backward optimizer & Adam & Adam & Adam\\ \hline
 $\text{wd}_{\text{fb}}$ & 3.993e-05 & 3.993e-05 & 6.169295107849636e-05 \\ \hline
 feedback pre-training epochs & 10 & 10 & 10\\ \hline
 feedback extra training epochs & 1 & 1 & 1\\ \hline
 scheduler & cosine & cosine & cosine \\ \hline
 scheduler eta min & 0.00001 & 0.00001 & 0.00001\\ \hline
 scheduler Tmax & 85 & 85 & 85\\ \hline
 scheduler interval/frequency & epoch/1 & epoch/1 & epoch/1 \\ \hline
 initialization & xavier normal & xavier normal & xavier normal \\ \hline
 batch size & 143 & 143 & 128\\ \hline
 epochs & 40 & 40 & 90 \\ 
 \bottomrule
  \end{tabular}
  \caption{Tuned s-DDTP LeNet hyperparameters}
  \label{table:sddtphparamslenet}
\end{table}

\begin{table}[h!]
\centering
  \begin{tabular}{p{4cm} p{4cm} p{4cm} p{4cm}}
    \toprule
    \multirow{2}{*}{Hyperparameter} & \multicolumn{3}{c}{Dataset} \\
      \multicolumn{1}{c}{ } &
      \multicolumn{1}{c}{ } & 
      \multicolumn{1}{c}{ } \\
      & {MNIST} & {F-MNIST} & {CIFAR-10} \\
      \midrule
 channels & [32, 64] & [32, 64] & [32, 64] \\ \hline
 target stepsize $\eta$ & 0.0428 & 0.01308 & 0.09983 \\ \hline
 $\beta_1$ & 0.9 & 0.9 & 0.9\\ \hline
 $\beta_2$ & 0.999 & 0.999 & 0.999 \\ \hline
 $\epsilon$ & [4.409e-06, 9.007e-07, 2.197e-05, 1.318e-05] & [2.376e-08, 4.795e-06, 4.672e-06, 1.663e-07] & [3.814e-05, 1.063e-07, 4.759e-07, 2.439e-06]\\ \hline
 $\epsilon_{\text{fb}}$ & 1.197e-08 & 2.491e-07 & 2.052e-07 \\ \hline
 $\beta_{1,\text{fb}}$ & 0.999 & 0.999 & 0.999 \\ \hline
 $\beta_{2,\text{fb}}$ & 0.999 & 0.999 & 0.999 \\ \hline
 $\sigma$ & [1.342e-05, 0.0002404, 2.927e-05, 0.0003338] & [0.003812, 0.00224, 0.0005647, 0.004229] & [0.000307, 3.066e-05, 7.908e-05, 0.0006653]\\ \hline
 activation & tanh & tanh & tanh\\ \hline
 $\text{lr}_{\text{f}}$ & [0.004434, 0.001448, 0.0006104, 0.001353] & [0.0003865, 0.00175, 0.001484, 0.0001489] & [0.0002289, 0.006166, 0.0001575, 5.573e-05] \\ \hline
 forward optimizer & Adam & Adam & Adam \\ \hline
$\text{wd}_{\text{f}}$ & 0 & 0 & 0 \\ \hline
$\text{lr}_{\text{fb}}$ &  0.0006438 & 0.0001795 & 0.001123 \\ \hline
 feedback training iterations & [49, 32, 54, 11] & [48, 42, 52, 22] & [24, 35, 36, 19] \\ \hline
 feedback activation & linear & linear & linear\\ \hline
 backward optimizer & Adam & Adam & Adam\\ \hline
 $\text{wd}_{\text{fb}}$ & 2.236e-05 & 0.000128 & 3.564e-06 \\ \hline
 feedback pre-training epochs & 0 & 0 & 0\\ \hline
 feedback extra training epochs & 0 & 0 & 0\\ \hline
 scheduler & cosine & cosine & cosine \\ \hline
 scheduler eta min & 0.00001 & 0.00001 & 0.00001\\ \hline
 scheduler Tmax & 85 & 85 & 85\\ \hline
 scheduler interval/frequency & epoch/1 & epoch/1 & epoch/1 \\ \hline
 initialization & xavier normal & xavier normal & xavier normal \\ \hline
 batch size & 138 & 141 & 190 \\ \hline
 epochs & 40 & 40 & 90 \\ 
 \bottomrule
  \end{tabular}
  \caption{Tuned p-DDTP LeNet hyperparameters}
  \label{table:pddtphparamslenet}
\end{table}

\end{document}